\documentclass{article}
\usepackage{ijcai23}

\usepackage{times}
\usepackage{soul}
\usepackage{url}
\usepackage[hidelinks]{hyperref}
\usepackage[utf8]{inputenc}
\usepackage[small]{caption}
\usepackage{graphicx}
\usepackage{amsmath}
\usepackage{amsthm}
\usepackage{booktabs}
\usepackage{algorithm}
\usepackage{algpseudocode,algorithmicx}  
\usepackage[switch]{lineno}

\newtheorem{theorem}{Theorem}
\newtheorem{lemma}[theorem]{Lemma}  

\usepackage{bbding}

\usepackage{subfigure}

\usepackage[capitalize,noabbrev]{cleveref}

\usepackage{amssymb}
\usepackage{mathtools}

\title{Theoretically Guaranteed Policy Improvement\\ Distilled from Model-Based Planning}

\author{
Chuming Li$^{1,2}$\footnote{Equal contribution.} \and
Ruonan Jia$^{1,3*}$\and
Jie Liu$^{1}$\and
Yinmin Zhang$^{1,2}$\and \\
Yazhe Niu$^1$\and
Yaodong Yang$^4$ \and 
Yu Liu$^1$ \and
Wanli Ouyang$^1$
\affiliations
$^1$Shanghai Artificial Intelligence Laboratory \quad 
$^2$University of Sydney\\
$^3$Tsinghua University \quad
$^4$Peking University
\emails
\{lichuming.lcm, jiaruonan97, liuyuisanai\}@gmail.com, \{liujie, niuyazhe\}@pjlab.org.cn,\\
\{yinmin.zhang, wanli.ouyang\}@sydney.edu.au, yaodong.yang@pku.edu.cn
}

\begin{document}

\maketitle

\begin{abstract}
Model-based reinforcement learning (RL) has demonstrated remarkable successes on a range of continuous control tasks due to its high sample efficiency. 
To save the computation cost of conducting planning online, recent practices tend to distill optimized action sequences into an RL policy during the training phase.
Although the distillation can incorporate both the foresight of planning and the exploration ability of RL policies, the theoretical understanding of these methods is yet unclear.
In this paper, we extend the policy improvement step of Soft Actor-Critic (SAC) by developing an approach to distill from model-based planning to the policy. We then demonstrate that such an approach of policy improvement has a theoretical guarantee of monotonic improvement and convergence to the maximum value defined in SAC. 
We discuss effective design choices and implement our theory as a practical algorithm---\textit{\textbf{M}odel-based \textbf{P}lanning \textbf{D}istilled to \textbf{P}olicy (MPDP)}---that updates the policy jointly over multiple future time steps.
Extensive experiments show that MPDP achieves better sample efficiency and asymptotic performance than both model-free and model-based planning algorithms on six continuous control benchmark tasks in MuJoCo.
\end{abstract}

\section{Introduction}
Model-based Reinforcement Learning (RL) has achieved great success on continuous control tasks \cite{gps,svg,steve,me-trpo,gobigger}. Model-based RL algorithms learn the true dynamics by fitting a model (usually a neural network) to the repeated interactions with the environment and use the model to generate imaginary data or perform online planning, which provides better sample efficiency than model-free RL \cite{dqn,ppo,td3,gem}. 

A typical kind of model-based RL algorithm performs online planning to optimize the future action sequence over a long time horizon, i.e., model-based planning \cite{gps,pets,poplin,latco}. However, model-based planning has two weaknesses. First, it can hardly be applied in real-time, because it needs to solve an optimization problem on each time step and cannot remember the solution for reuse in the future similar states \cite{poplin}. Second, it only optimizes the maximum of the reward sum over the future states, rather than the trade-off between exploration and exploitation, which limits the ability to discover diverse states and better policies \cite{gps}. To reduce the time consumption during the application and incorporate the foresight of planning and the exploration ability of RL, some recent works distill the result of model-based planning into an RL policy \cite{gps,poplin}. Specifically, POPLIN uses the cross entropy method (CEM) \cite{cem} to optimize the action planning and uses behavior cloning to distill the planning result into the policy network. However, some essential theoretical properties of such kind of distillation are not well-understood, i.e., \textbf{(1)} whether the distilled policy achieves a higher value than the old policy; \textbf{(2)} whether the distilled policy has a guarantee of convergence to the optimal policy; \textbf{(3)} whether the distilled policy incorporates the foresight of planning and achieves a higher value than the original model-free policy update.

In this paper, we theoretically analyze the problems mentioned above. We choose Soft Actor-Critic (SAC) \cite{sac} as the RL component of our analysis due to its state-of-the-art performance in both model-free and model-based paradigms. Originally, the policy improvement of SAC is a one-step optimization. 
We first define a planning problem by extending the one-step optimization of SAC under the model-based paradigm to a multi-step optimization problem of action planning. For each state $\boldsymbol{s}_t$, the optimal planning solution returns a policy $\pi_{\boldsymbol{s}_t}^H$ defined on a horizon of states $\boldsymbol{s}_{t:t+H-1}$ starting from $\boldsymbol{s}_t$.
Then, we propose a simple approach to distill the solution of the above multi-step optimization to the policy, which is an extended form of the policy improvement of SAC. This approach reserves the returned policy $\pi_{\boldsymbol{s}_t}^H(\cdot|\boldsymbol{s}_t)$ for the first state $\boldsymbol{s}_t$ and discards the returned policy $\pi_{\boldsymbol{s}_t}^H(\cdot|\boldsymbol{s}_{t+1:H-1})$ for the future states.

\begin{table*}[ht]
  \centering
  \begin{tabular}{ccccc}
    \toprule
    Algorithms  &Ensemble Dynamics  &Multiple Horizon  &Regularization  &Planning Theorem\\ 
    \midrule
    SAC\cite{sac}    &\XSolidBrush    &\XSolidBrush    &\XSolidBrush   &\XSolidBrush\\
    MBPO\cite{mbpo}    &\CheckmarkBold    &\XSolidBrush    &\XSolidBrush   &\XSolidBrush\\
    POPLIN\cite{poplin}&\CheckmarkBold    &\CheckmarkBold  &\XSolidBrush   &\XSolidBrush\\
    M2AC\cite{m2ac}    &\CheckmarkBold    &\CheckmarkBold  &\XSolidBrush   &\XSolidBrush\\
    MPDP(our work)     &\CheckmarkBold    &\CheckmarkBold  &\CheckmarkBold     &\CheckmarkBold\\
    \bottomrule
  \end{tabular}
  \caption{Key features of different model-free and model-based algorithms.}
  \label{tab:component}
\end{table*}

Afterwards, we derive the theoretical result that the extended policy improvement is promising to achieve a higher return and lead the policy to converge to the optimal policy.
Thus the extension incorporates the farsight planning and has the potential to improve remarkably upon original one-step policy improvement.
Furthermore, to develop a practical algorithm, we discuss the solver of the defined multi-step optimization and design regularization to reduce the model error. Based on the above theory and discussion, we propose a new model-based RL algorithm, \textit{\textbf{M}odel-based \textbf{P}lanning \textbf{D}istilled to \textbf{P}olicy (MPDP)}. Compared to POPLIN, which uses behavior cloning for distillation and realizes the stochastic exploration via the CEM sampling, MPDP utilizes a distillation approach with theoretically guaranteed improvement and inherits the stochastic exploration of SAC, thus has a naturally strong ability to explore better policies.
For illustrating the effectiveness of MPDP, a thorough component comparison of relevant algorithms is given in \cref{tab:component}.

\textbf{Summary of Contributions:}
\textbf{(1)} We propose a model-based extended policy improvement method, which utilizes model-based planning to distill RL policy and model regularization to reduce the impact of model errors.
\textbf{(2)} We demonstrate that our method has a theoretical guarantee of monotonic improvement and convergence. And we theoretically analyze how the planning horizon affects policy improvement.
\textbf{(3)} Experimental results empirically show that MPDP achieves better sample efficiency and asymptotic performance than state-of-the-art model-free and model-based planning algorithms on the MuJoCo \cite{mujoco}. 
\section{Related Work}

\paragraph{Model-based Reinforcement Learning.}
Model-based reinforcement learning methods show a promising prospect for real-world decision-making problems due to their data efficiency. However, learning an accurate model is challenging, especially in complex environments. Many papers \cite{pets,me-trpo,mbpo,pdml} commonly use ensemble probabilistic networks to construct uncertainty-aware environment models.

The previously proposed model-based methods \cite{mve,steve,emc-ac,vagram} allow the model rollout to a fixed depth, and value estimations are split into a model-based reward and a model-free value. To guarantee the monotonic improvement, the recent work 
\cite{slbo} builds a lower bound of the expected reward and then maximizes the lower bound jointly over the policy and the model. Furthermore, model-based policy optimization \cite{mbpo} utilizes short model-generated rollouts to do policy improvement and evaluation, and also provides a guarantee of monotonic improvement.

Current model-based RL mainly focuses on better model usage. For example, M2AC \cite{m2ac} implements a masking mechanism based on the model’s uncertainty to decide whether its prediction should be used or not. Another line of works \cite{gps,svg} aims to exploit the differentiability of the learned model in model-based RL. Model-augmented actor-critic \cite{maac} uses the path-wise derivative of the learned model and policy across future time steps.
Our work estimates value function by utilizing the model error as regularization.

\paragraph{Model-based Planning.}
Many recent papers on deep model-based RL \cite{pets,VisualForesight,mpc} optimize the future action trajectories over a given horizon starting from the current state, which is usually referred as model-based planning.
Model predictive control \cite{mpc} is a common control approach for model-based planning. It frequently solves the action planning over a limited horizon and conducts the first action on the environment.
Random Shooting optimizes the action sequence among the randomly generated candidates to maximize the expected reward under the learned dynamic model, and PETS \cite{pets} uses the cross entropy method \cite{cem} to improve the efficiency of the random search. However, shooting methods usually rely on the local search in the action space and are not effective on high-dimension environments. 
To solve this problem,  the latest work \cite{latco} utilizes the collocation-based planning in a learned latent space.
In contrast, we extend the policy improvement step of SAC to distill from model-based planning to the policy, which reduces the cost in the deployment phase.

In addition, some recent works distill the result from model-based policy planning into an RL policy. POPLIN \cite{poplin} formulates action planning at each time step as an optimization problem w.r.t. the parameters of the policy network, and uses behavior cloning to distill the resulted action into the policy network. GPS \cite{gps2,gps} uses KL divergence to minimize the distance between the policy and the planning result. However, the essential theoretical properties of such distillation are not well-understood.
Instead, we propose an algorithm to improve the policy with the solution of model-based planning over multiple time steps, and give the theoretical guarantee of its improvement and convergence.

\paragraph{Actor-Critic Methods.}
Actor-critic algorithms are typically derived from policy iteration, which alternates between policy evaluation and policy improvement.
Deep deterministic policy gradient \cite{ddpg} is a common model-free actor-critic method, however, the critic is usually overestimated to predict Q value, which leads to the worse policy.
Moreover, twin delayed deep deterministic policy \cite{td3} mainly utilizes the clipped double Q learning to alleviate the above overestimation.
SAC \cite{sac,taec} is the SOTA algorithm of policy learning under the model-based paradigm. In the framework of SAC, the actor aims to maximize expected reward with entropy and the critic evaluates the expected cumulative reward with entropy. Due to the splendid performance of SAC, we choose it as the RL instance to prove the theoretical properties, by distilling the planning into an RL policy.
\section{Preliminaries}

\subsection{Notation}
We consider continuous control tasks which can be formulated as infinite-horizon Markov Decision Processes (MDP) $(\mathcal{S},\mathcal{A},p,r,\gamma)$, where the state space $\mathcal{S}$ and the action space $\mathcal{A}$ are both continuous. State transition $p:\mathcal{S} \times \mathcal{A} \times \mathcal{S} \rightarrow \mathbb{R}^+$ and $r:\mathcal{S} \times \mathcal{A} \rightarrow \mathbb{R}$ are the dynamics of the environment and the reward function, respectively. $\gamma$ is the discount factor. Additionally, we define $\pi(\boldsymbol{a}|\boldsymbol{s}):\mathcal{S} \times \mathcal{A} \rightarrow \mathbb{R}^+$ as the RL policy on the state $\boldsymbol{s}$, with $Q(\boldsymbol{s},\boldsymbol{a})$ and $V(\boldsymbol{s})$ as the corresponding value functions.

\subsection{Soft Actor-Critic}
Soft Actor-Critic(SAC) \cite{sac} develops a maximum entropy objective to incentivize the policy to explore more widely, which is the discounted sum of both the reward and the entropy, formalized as:
\begin{gather}
\label{eq:sac_obj}
    J_{\boldsymbol{s_t}}(\pi) = 
    \mathbb{E}_{\boldsymbol{a_t} \sim \pi}
    \left[
        \sum_{t=0}^{\infty}
            \gamma^t 
            \cdot     [r(\boldsymbol{s}_{t},\boldsymbol{a}_{t}) 
    - \alpha \cdot log\pi(\boldsymbol{a}_{t}|\boldsymbol{s}_{t})]
    \right].
\end{gather}

The coefficient $\alpha$ balances the importance of
the reward and entropy, and hence controls the exploration of the policy. we omit $\alpha$ in the rest of this paper for simplicity.
The policy evaluation of SAC is based on the maximum entropy objective, i.e., the value function $Q$ and $V$ also contain the discounted sum of the entropy over the subsequent states. The Bellman backup operator $\mathcal{T}^{\pi}$ of SAC is given by:
\begin{gather}
\label{eq:sac_evaluation}
\mathcal{T}^{\pi}Q(\boldsymbol{s}_{t},\boldsymbol{a}_{t}) 
= r(\boldsymbol{s}_{t},\boldsymbol{a}_{t}) 
+ \gamma \cdot V(\boldsymbol{s}_{t+1}), \\
V(\boldsymbol{s}_{t}) = 
\mathbb{E}_{\boldsymbol{a}_{t} \sim \pi}
    \left[
    Q(\boldsymbol{s}_{t},\boldsymbol{a}_{t}) -
    log\pi(\boldsymbol{a}_{t}|\boldsymbol{s}_{t})
    \right].
\end{gather}

In the policy improvement step of SAC, the new policy optimizes the $V(\boldsymbol{s}_{t})$ on each state $\boldsymbol{s}_{t}$:
\begin{gather}
\pi_{new}(\cdot|\boldsymbol{s}_{t}) = \arg \max_{\pi} \mathbb{E}_{\boldsymbol{a}_{t} \sim \pi}
    \left[
    Q^{\pi_{old}}(\boldsymbol{s}_{t},\boldsymbol{a}_{t}) -
    log\pi(\boldsymbol{a}_{t}|\boldsymbol{s}_{t})
    \right].
\end{gather}
We reformulate the objective as:
\begin{gather}
\begin{aligned}
\pi_{new}(\cdot|\boldsymbol{s}_{t}) = \arg \max_{\pi} \mathbb{E}_{\boldsymbol{a}_{t} \sim \pi}
    &\left[   \right.
    r(\boldsymbol{s}_{t},\boldsymbol{a}_{t}) 
 -log\pi(\boldsymbol{a}_{t}|\boldsymbol{s}_{t})
  \\
 &+ \gamma \cdot
V^{\pi_{old}}(\boldsymbol{s}_{t+1})
   \left. \right].
\label{eq:sac_improvement}
\end{aligned}
\end{gather}

This objective leads the new policy to optimize the modified reward $r(\boldsymbol{s}_{t},\boldsymbol{a}_{t}) -log\pi(\boldsymbol{a}_{t}|\boldsymbol{s}_{t})$ only on the current state $\boldsymbol{s}_{t}$ w.r.t. $\boldsymbol{a}_{t}$, with the subsequent states following the old policy $\pi_{old}$, which is myopic under the model-based paradigm, because the dynamics of the environment can be approximated by the environment model, which enables the joint optimization of actions over multiple future time-steps.

\subsection{Environment Model}
A common setting used in model-based RL is model ensemble \cite{pets,me-trpo,mbpo,slbo,m2ac}, where an ensemble of models learn the distribution of the transitions from historical interactions. Typically, the models are parametric function approximators $p_{1:K}(\cdot|\boldsymbol{s},\boldsymbol{a})$ and are trained via maximum likelihood: $\sum_{i=1}^{K}\mathbb{E}\left[log( p_i(\boldsymbol{s}_{t+1}|\boldsymbol{s}_t,\boldsymbol{a}_t)) \right]$.
\section{Distillation from Planning into Policy}
In this section, we propose an approach to distilling the solution of model-based planning into the policy, which is a multi-step extension of the original policy improvement of SAC. We will first derive this extension. Then, we will verify its theoretical properties and advantages. Finally, based on our theory, we will develop a practical reinforcement learning algorithm by discussing the essential design choices in the next section.

\subsection{Multi-step Optimization}
\label{multi_optim}
The policy improvement of SAC optimizes the trade-off between the expected cumulative reward and entropy only with regard to the action distribution on the current time-step $\boldsymbol{s}_t$, with the future states $\boldsymbol{s}_{t+1:\infty}$ following the old policy $\pi_{old}$, formalized in \cref{eq:sac_improvement}. Under the model-based paradigm, we assume that the true dynamics of the environment is accessible. Because we can always obtain a dynamic model with a lower generalization error \cite{me-trpo,mbpo}, as the training proceeds. This assumption enables us to quantify the expected future state and the according reward and entropy with regard to the future action sequence over a given horizon $H$, and derive a more foresighted optimization form than the original SAC. Specifically, we extend the one-step optimization in \cref{eq:sac_improvement} to a multi-step optimization problem of the action planning over $H$ steps based on the environment model, with the objective $J_{\boldsymbol{s}_t}^H(\pi)$ on the state $\boldsymbol{s}_t$ defined as: 

\begin{gather}
    J_{\boldsymbol{s}_t}^H(\pi) 
    =
    \mathbb{E}_{\boldsymbol{a_t} \sim \pi}
    \left[
        \sum_{i=0}^{H-1}
            \gamma^i 
            \cdot r^{\pi}(\boldsymbol{s}_{t+i},\boldsymbol{a}_{t+i})
            + V^{\pi_{old}}(\boldsymbol{s}_{t+H})
    \right], \label{eq:objective}\\
    r^{\pi}(\boldsymbol{s}_{t+i},\boldsymbol{a}_{t+i}) 
    = 
    r(\boldsymbol{s}_{t+i},\boldsymbol{a}_{t+i}) 
    - log\pi(\boldsymbol{a}_{t+i}|\boldsymbol{s}_{t+i}).
\end{gather}
Here $H$ is the planning horizon, $\pi$ is the policy only defined on $\boldsymbol{s}_t$ and its subsequent $H-1$ steps. $r^{\pi}(\boldsymbol{s},\boldsymbol{a})$ is the sum of the reward and the logarithmic likelihood, which inherits the maximum entropy objective of SAC. Specifically, when $H=1$, this objective degenerates to that of SAC. 

\subsection{Extended Policy Improvement}
The improvement property of distillation from planning into an RL policy has not been well discussed. Another work\cite{maac} proves that the solution of action planning achieves a higher value, but it does not develop a  distillation approach to obtain a policy $\pi_{new}$ with provably higher value $V^{\pi_{new}}(\boldsymbol{s}_{t})$, i.e., a policy with higher cumulative rewards.
In this section, we propose a distillation approach, also an extended form of the original policy improvement step in SAC, based on the multi-step optimization.
We will show that the proposed extended policy improvement provably achieves a new policy with a higher value than the old policy with respect to the maximum entropy target \cref{eq:sac_obj} defined in SAC.

\paragraph{Distillation.} We use $\pi_{\boldsymbol{s}_t}^H$ to denote the optimal solution of $J_{\boldsymbol{s}_t}^H(\pi)$. After the policy improvement, we define the new policy $\pi_{new}(\cdot|\boldsymbol{s}_t)$ as $\pi_{\boldsymbol{s}_t}^H(\cdot|\boldsymbol{s}_t)$, i.e., although $\pi_{\boldsymbol{s}_t}^H$ is define on $H$ steps of states $\boldsymbol{s}_{t:t+H-1}$, we only adopt the policy $\pi_{\boldsymbol{s}_t}^H(\cdot|\boldsymbol{s}_t)$ of the current state $\boldsymbol{s}_t$ and discard the policy $\pi_{\boldsymbol{s}_t}^H(\cdot|\boldsymbol{s}_{t+1:t+H-1})$ over the following states. 

\paragraph{Improvement.} We present the improvement property of this distillation in \cref{lem:improve}. Please note that \cref{lem:improve} is a more general multi-step extension of  the Lemma 2\footnote{https://arxiv.org/pdf/1801.01290.pdf}  in SAC \cite{sac}. Our result reveals that, if we optimize the policy jointly over a horizon starting from each state $\boldsymbol{s}_t$ and only adopt the optimal policy on the first state $\boldsymbol{s}_t$, the resulting new policy has a monotonic improvement. Specifically, when $H=1$, \cref{lem:improve} degenerates to the Lemma 2 in SAC (see Appendix A. for more details).
\begin{lemma} 
\label{lem:improve}
Let $\pi_{\boldsymbol{s}_t}^H$ be the optimizer of the optimization objective of \cref{eq:objective}. 
When the new policy $\pi_{new}(\cdot|\boldsymbol{s}_t) = \pi_{\boldsymbol{s}_t}^H(\cdot|\boldsymbol{s}_t)$,
$V^{\pi_{new}}(\boldsymbol{s}_t) \geq V^{\pi_{old}}(\boldsymbol{s}_t)$ for all $\boldsymbol{s}_t \in S $.
\end{lemma}

\subsection{Policy Convergence}
The monotonic increasing property of our extended form is crucial, because it facilitates the derivation of the proposition that this form will provably converge to the optimal maximum entropy policy defined in SAC. We present the result in \cref{thm:policy_conv}.
\begin{theorem}
\label{thm:policy_conv}
Let $\pi_0$ be any initial policy. Assuming $|\mathcal{A}|<\infty$, if the policy evaluation in \cref{eq:sac_evaluation} and the policy improvement with the objective in \cref{eq:objective} are alternatively carried out, $\pi_0$ converges to a policy $\pi_*$, with $V^{\pi_{*}}(\boldsymbol{s}_t) \geq V^{\pi}(\boldsymbol{s}_t)$ for any $\boldsymbol{s}_t \in S $.
\end{theorem}

\subsection{The Effect of Planning Horizon}
We have shown that the proposed extension of policy improvement, based on optimization of the action planning over multiple time steps, can always lead to a higher value via the developed distillation, which is guaranteed to converge to the optimal policy. In this section, we will discuss another problem: does the extended form of policy improvement incorporate the farsight of planning and benefit SAC? Or more generally, does a larger planning horizon $H$ always result in a better value? 

Unfortunately, there exist some special cases where a larger $H$ leads to a smaller value due to a bad initial policy $\pi_{old}$. 
Although a larger $H$ is not equivalent to a higher value, we can still show the potential advantage of increasing $H$ in two aspects. 

\textbf{(1)} A larger horizon results in a higher optimization objective defined in \cref{eq:objective}, as formalized in \cref{lem:monotone}.

\begin{lemma}
\label{lem:monotone}
Let $\pi_{\boldsymbol{s}_t}^H$ and $\pi_{\boldsymbol{s}_t}^{H+1}$ be the optimal solution of $J_{\boldsymbol{s}_t}^H(\pi)$ and $J_{\boldsymbol{s}_t}^{H+1}(\pi)$. Then $J_{\boldsymbol{s}_t}^{H+1}(\pi_{\boldsymbol{s}_t}^{H+1}) \geq J_{\boldsymbol{s}_t}^{H}(\pi_{\boldsymbol{s}_t}^{H})$ for all $H\geq 1$ and $\boldsymbol{s}_t \in S$.
\end{lemma}

\textbf{(2)} Although the resulting policy does not have a value monotonically increasing with $H$, we can prove that $\pi_{new}$ converges to the optimal policy as $H$ increases, which is formalized in \cref{thm:optim}.

\begin{theorem}
\label{thm:optim}
Let $\pi_{\boldsymbol{s}_t}^H$ be the optimal solution of $J_{\boldsymbol{s}_t}^H(\pi)$, and $\pi_{new}(\cdot|\boldsymbol{s}_t)=\pi_{\boldsymbol{s}_t}^H(\cdot|\boldsymbol{s}_t)$. $\pi_{*}$ denotes the optimal policy. As $H$ increases, $V^{\pi_{new}}$ and $J_{\boldsymbol{s}_t}^H(\pi_{\boldsymbol{s}_t}^H)$ converge to $V^{\pi_{*}}$ for all $\boldsymbol{s}_t \in S$. Specifically, $V^{\pi_{new}} \geq J_{\boldsymbol{s}_t}^H(\pi_{\boldsymbol{s}_t}^H) \geq V^{\pi_{*}}(\boldsymbol{s}_{t}) 
-    
\frac{\gamma^H
\cdot r^{max}}{1-\gamma}$ with $r^{max}$ the maximum of $r^{\pi}(\boldsymbol{s},\boldsymbol{a})$ over all $\pi$ and $(\boldsymbol{s},\boldsymbol{a})\in |\mathcal{S}|\times|\mathcal{A}|$.
\end{theorem}

Starting from \cref{thm:optim}, it can be naturally derived that, we can always find a larger $\hat{H}$ than $H$, which results in a policy with a larger value. We formalize this conclusion as \cref{thm:exist}.

\begin{theorem}
\label{thm:exist}
Let $\pi_{\boldsymbol{s}_t}^H$ be the optimal solution of $J_{\boldsymbol{s}_t}^H(\pi)$, and $\pi_{new}^H(\cdot|\boldsymbol{s}_t)=\pi_{\boldsymbol{s}_t}^H(\cdot|\boldsymbol{s}_t)$. There exists another $\hat{H} > H$, with $V^{\pi_{new}^{\hat{H}}}\geq V^{\pi_{new}^{H}}$ for all $\boldsymbol{s}_t \in S$, assuming $|\mathcal{S}|<\infty$.
\end{theorem}

\begin{proof}
According to \cref{thm:optim}, we can always find a $\hat{H}$ with $V^{\pi_{*}}-V^{\pi_{new}^{\hat{H}}} \leq V^{\pi_{*}}-V^{\pi_{new}^{H}}$ on all states, which means $V^{\pi_{new}^{\hat{H}}}\geq V^{\pi_{new}^{H}}$.
\end{proof}
\section{Implementation}
According to the above theory, the proposed extended policy improvement via planning over multiple time steps can also guarantee value improvement and convergence to the optimal policy. And the increase of planning horizon has the potential to get a better new policy. In this section, we discuss some essential design choices for distilling the model-based planning into SAC \cite{sac}. We further propose a practical algorithm, \textit{\textbf{M}odel-based \textbf{P}lanning \textbf{D}istilled to \textbf{P}olicy (MPDP)}, under the model-based paradigm.
There are two essential issues in the design of MPDP, \textbf{(1)} how to solve the objective in \cref{eq:objective}, and \textbf{(2)} how to reduce the bias introduced by the generalization error of the environment model.

\subsection{Solver} 
Solving the proposed objective defined by \cref{eq:objective} is a model-based planning problem, which has been widely discussed in many prior works \cite{latco,pets,poplin}. We roughly divide the current solvers into two categories, sample-based methods and gradient-based methods.

Sample-based methods typically include random shooting and cross-entropy method (CEM) \cite{cem}. However, sample-based methods are usually inefficient in complex high-dimensional tasks.
Gradient-based methods include gradient optimization and collocation method \cite{latco}, which optimize with reward to the action sequence and backpropagate the gradient to all actions in the sequence.
Both gradient optimization and collocation methods suit our formulation due to their accessibility of the gradient. We can develop a practical algorithm based on both of them. We observe that they perform comparably on the MuJoCo benchmark in our early-stage experiments.

With the above discussion, we choose gradient optimization as our solver, because it naturally suits the framework of SAC and achieves comparable performance without introducing extra hyperparameters and computational cost compared to the collocation method.

\begin{algorithm}[ht]
   \caption{Farsighted Policy Improvement}
   \label{alg:improvement}
\begin{algorithmic}[1] 
   \Require state batch $B$, policy networks $\pi_{0:H_{max}-1}$, dynamic models $p_{1:K}$, threshold $u_T$, coefficient $\alpha$ and $\beta$
   \For{$\boldsymbol{s}$ {\bfseries in} $B$}
   \State $\boldsymbol{s}_0 = \boldsymbol{s}$, $J=0$
   \For{$t=0:H_{max}-1$}
        \State Sample $\boldsymbol{a}_t \sim \pi_{t}$
        \State Predict $\boldsymbol{s}_{t+1} \sim p_{1:K}(\boldsymbol{s}_{t+1},\boldsymbol{a}_{t})$
   \If{$u(\boldsymbol{s}_t,\boldsymbol{a}_t) \geq u_T$ or $\boldsymbol{s}_{t+1}$ is a terminal state}
   \State $J=J+\gamma^{t+1} \cdot V(\boldsymbol{s}_{t+1})$
   \State break
   \EndIf
  \State \begin{small}$
  J \!=\! J \!+\! \gamma^{t} \cdot [ r(\boldsymbol{s}_{t},\boldsymbol{a}_{t})
    \!-\! \alpha \cdot log\pi(\boldsymbol{a}_{t}|\boldsymbol{s}_{t})
    \!-\! \beta \cdot u(\boldsymbol{s}_{t},\boldsymbol{a}_{t}) ]
          $\end{small}
   \EndFor
   \EndFor
   \State Update $\pi_{0:H_{max}-1}$ with the mean of $\nabla_{\boldsymbol{a}_{0:H_{max}-1}}J$
\end{algorithmic}
\end{algorithm}

\subsection{Model Regularization} 
The bias resulting from the environment model's generalization error raises two issues for consideration. 
First, although increasing the planning horizon has the potential of resulting in a higher value theoretically, we must consider the trade-off between the bias of $Q^{\pi_{old}}$ and the environment model. A larger $H$ introduces more model bias but reduces the bias of $Q^{\pi_{old}}$. 
Second, we need to avoid the update of the policy towards the area where the model has high generalization error, because this will result in a sub-optimal solution and the gradients of the environment model at those unseen state-action pairs $(s, a)$ are unsupervised and not numerically stable, i.e., applying the environment model iteratively for many time steps may lead to gradient explosion \cite{latco}.

Both the two issues need the estimation of the model error, which has been well discussed in prior works. In this paper, we use One-vs-Rest (OvR) \cite{m2ac}, a simple method to estimate model errors. OvR learns multiple dynamic models and uses the KL divergence between models as an estimator of model error, which is formalized as:
\begin{equation}
\setlength{\abovedisplayskip}{5pt}
\setlength{\belowdisplayskip}{5pt}
    u(\boldsymbol{s},\boldsymbol{a}) = \sum_{i=1}^{K}D_{KL}[p_{i}(\cdot|\boldsymbol{s},\boldsymbol{a}) \Vert p_{-i}(\cdot|\boldsymbol{s},\boldsymbol{a})].
\end{equation}
Here $p_{i}(\cdot|s,a)$ is the predicted distribution of the one model and $p_{-i}(\cdot|s,a)$ is the mean of the rest models' prediction. 

Based on OvR, we develop two methods separately for the above two issues. First, we use adaptive horizons for trajectories starting from different states. The planning stops when a trajectory generates a state-action pair which has a model error larger than a pre-defined threshold. Secondly, we develop an additional regularization of model error, which adds the model error estimated by OvR on our objective \cref{eq:objective}. This regularization directs the final solution to the area where the environment model is more believable and reduces both the numerical instability and the model error. Specifically, we add the estimation $u(\boldsymbol{s},\boldsymbol{a})$ on the original reward $r^{\pi}(s,a)$ as a regularization, and re-formalize \cref{eq:objective} as:
\begin{equation}
    J_{\boldsymbol{s}_t}^{H,u}(\pi) 
    = 
    \mathbb{E}_{\boldsymbol{a} \sim \pi}
    \left[
        \sum_{i=0}^{H-1}
            \gamma^i 
            \cdot r^{\pi,u}(\boldsymbol{s}_{t+i},\boldsymbol{a}_{t+i})
            + V^{\pi_{old}}(\boldsymbol{s}_{t+H})
    \right], \label{eq:objective_reg}
\end{equation}
\begin{gather}
\begin{aligned}
    r^{\pi,u}(\boldsymbol{s}_{t+i},\boldsymbol{a}_{t+i}) 
    =
    &r(\boldsymbol{s}_{t+i},\boldsymbol{a}_{t+i}) 
    - log\pi(\boldsymbol{a}_{t+i}|\boldsymbol{s}_{t+i})
    \\
    &- \beta \cdot u(\boldsymbol{s}_{t+i},\boldsymbol{a}_{t+i}).
    \label{eq:beta_obj}
\end{aligned}
\end{gather}

\begin{algorithm}[ht]
   \caption{Model-based Planning Distilled to Policy}
   \label{alg:MPDP}
\begin{algorithmic}[1]
   \State Initialize data buffer $D=\emptyset$, dynamic models $p_{1:K}$, policy networks $\pi_{0:H_{max}-1}$, value networks $Q$ and $V$
   \Repeat
   \State Collect data from real environment with policy $\pi_0$: $D \leftarrow D\cup {(s,a,r,s')}$
   \State Train ensemble models $p_{1:K}$ on $D$
   \State Sample a batch $B$ from $D$
   \State Update $Q$ and $V$ with $B$ as in SAC
   \State Update $\pi_{0:H_{max}-1}$ by \cref{alg:improvement}.
   \Until{Convergence}
\end{algorithmic}
\end{algorithm}

\subsection{Model-based Planning Distilled to Policy}

\begin{figure*}[!ht]
\vskip 0.2in
\begin{center}
\centerline{\includegraphics[width=1.05\textwidth]{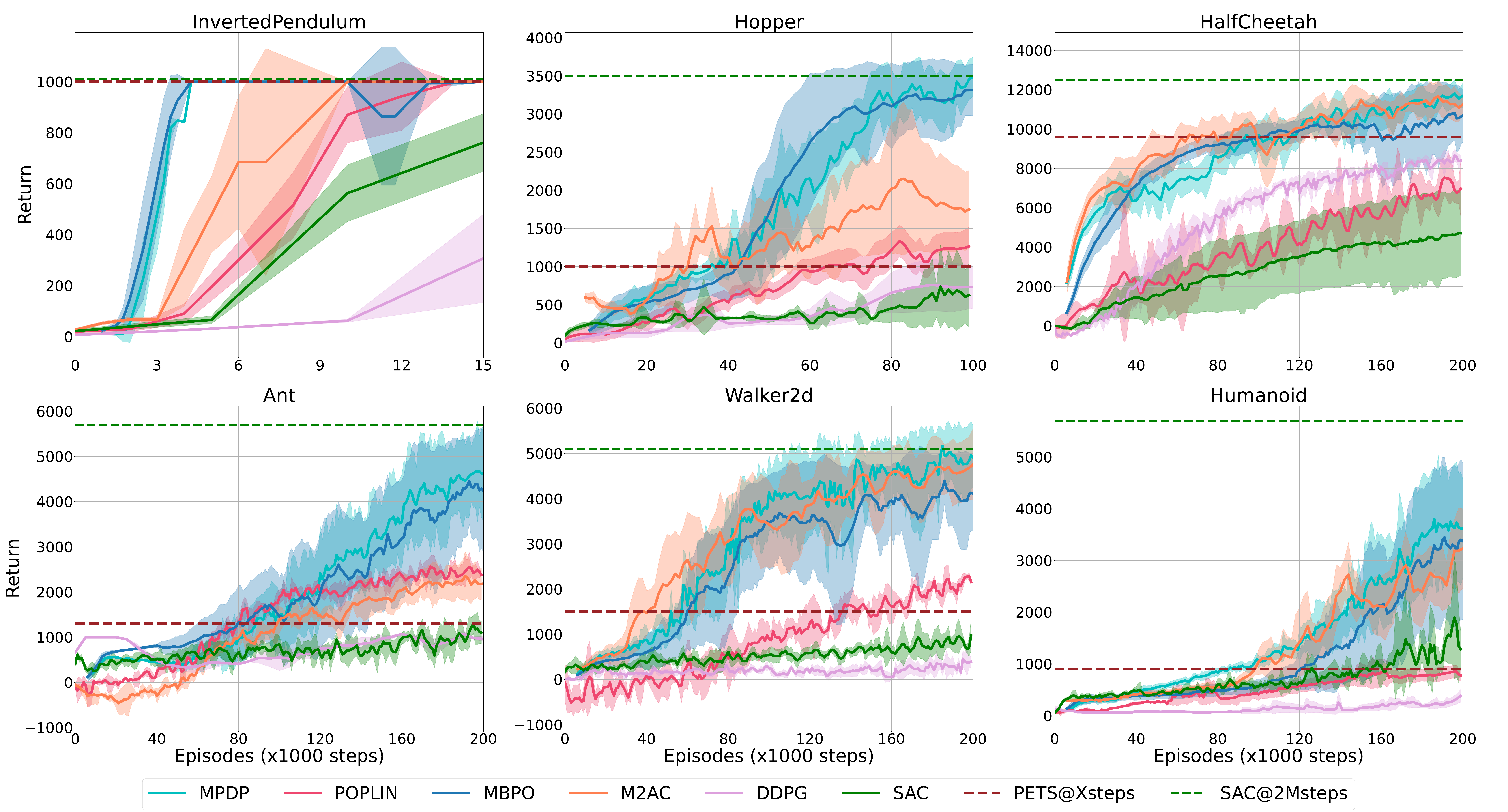}}
\caption{Performance curves for our method (MPDP) and baselines on MuJoCo continuous control benchmarks. Solid lines depict the mean of four trials and shaded regions correspond to standard deviation among trials. The dashed lines indicate the asymptotic performance of PETS at the corresponding training steps (15k steps for InvertedPendulum, 100k steps for Hopper, and 200k steps for the other tasks) and SAC at 2M steps.}
\label{fig:result}
\end{center}
\vskip -0.2in
\end{figure*}

We conclude our extended policy improvement in \cref{alg:improvement}. The algorithm processes a batch of states at each iteration and the model rollouts states until the task terminates, that is to say, the pair of $(\boldsymbol{s}_t, \boldsymbol{a}_t)$ has a larger model error than the threshold $u_T$, or the rollout reaches the max horizon $H_{max}$. And we maintain the policy networks $\pi_{0:H_{max}-1}$ at $H$ time steps. The policy networks generate the actions for each step and are updated jointly in our extended improvement step. After the model rollouts, the policy networks $\pi_{0:H_{max}-1}$ are updated with the gradients to the action sequence. The complete algorithm is described in \cref{alg:MPDP}. The method alternates among using the policy $\pi_{0}$ on the first step to interact with the environment, training an ensemble of models, and updating the policy with policy evaluation and our extended policy improvement.

\section{Experiment}
Our experiment goal is to investigate the following questions: \textbf{(1)} How the sample efficiency and the asymptotic performance of MPDP compared to state-of-the-art(SOTA) model-based planning algorithms? \textbf{(2)} How the proposed extended policy improvement and the design choices affect the performance?


\subsection{Comparison}

\paragraph{Baseline.} 
In this section, we focus on understanding how well MPDP performs compared to SOTA model-based planning algorithms. We choose PETS \cite{pets}, which uses CEM to perform model-based action planning; and POPLIN \cite{poplin}, which extends CEM from action space to the domain of policy network parameters and distills the planning results into the policy with behavior cloning. 
Additionally, we compare our proposed approach to the SOTA model-free methods and model-based methods without planning. For model-free algorithms, we compare to SAC \cite{sac} and DDPG \cite{ddpg}, which are the two competitive policy learning algorithms. 
For model-based RL, we choose MBPO \cite{mbpo} and M2AC \cite{m2ac}, which are the previous SOTA model-based baselines. 
MPDP, PETS, POPLIN, MBPO and M2AC share the same model architecture. The implementation details of our method are in Appendix B.

\paragraph{Results.}  
The performance curves on all six environments of MuJoCo are shown in \cref{fig:result}. It demonstrates that MPDP significantly outperforms the SOTA model-based planning algorithms (PETS and POPLIN) on both sample efficiency and asymptotic performance. For example, on the highly dimensional Ant task, MPDP’s performance at 140k steps is equivalent to that of POPLIN at 200k steps.

Further, the results in \cref{fig:result} reveal that MPDP achieves much higher convergence speed than the SOTA of model-free algorithms (SAC and DDPG) on the all tasks and obtains comparable asymptotic performance, which also validates that incorporating our extended policy improvement benefits a lot. We also observe that MPDP achieves better performance than the SOTA model-based algorithms, MBPO and M2AC on some complex tasks like Humanoid, and is comparable to them on the rest of tasks.

\subsection{Ablation Study}
In this section, we conduct a series of ablation studies on MPDP to investigate the effect of the designed adaptive horizon and regularization on the model error. We choose the Hopper task in the MuJoCo for the experiments.

\paragraph{Horizon.}
To verify that our method can really adapt the horizon to the model error, i.e. the adaptive horizon does not fall into a very small range and increases as the model generalizes better, we profile the average horizon of MPDP during the training on Hopper with different error threshold $u_T$ in \cref{fig:horizon_length}. As shown in the curves, the horizon grows from 2 to 12 as the training proceeds, where the model becomes more accurate in \cref{fig:beta_error}. It also proves that MPDP does not degenerate to SAC.

\begin{figure}[!h]
\begin{center}
\centerline{\includegraphics[width=0.47\textwidth]{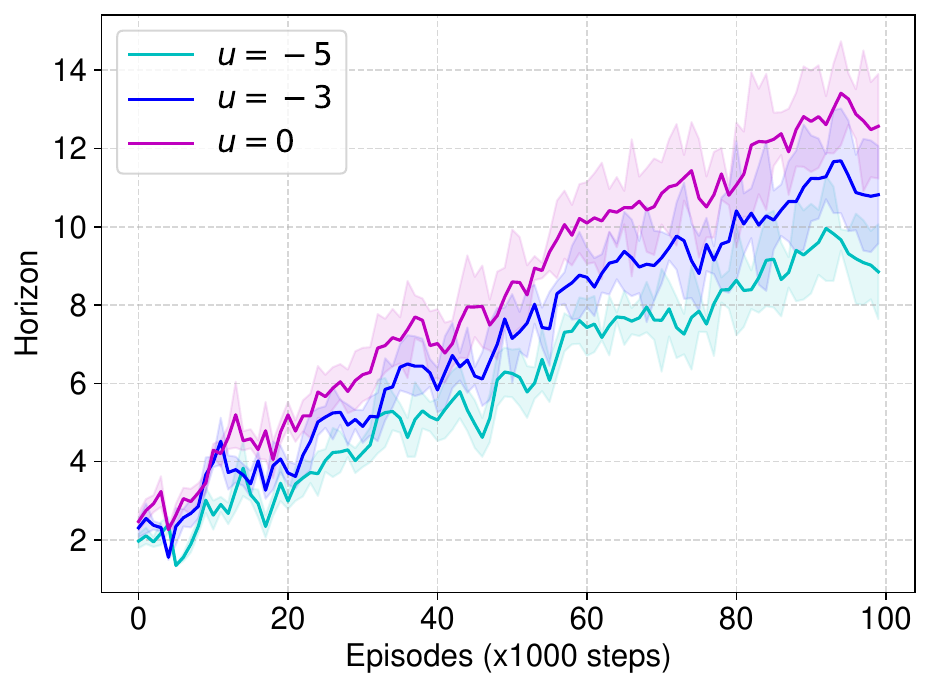}}
\caption{This figure demonstrates the length of the adaptive horizon of MPDP. The solid lines denote the average horizon length evaluated on each training batch. As the interactions accumulate, the model generalizes better and our method rapidly adapts to longer horizons.}
\label{fig:horizon_length}
\end{center}
\end{figure}

\begin{figure}[!h]
\begin{center}
\centerline{\includegraphics[width=0.47\textwidth]{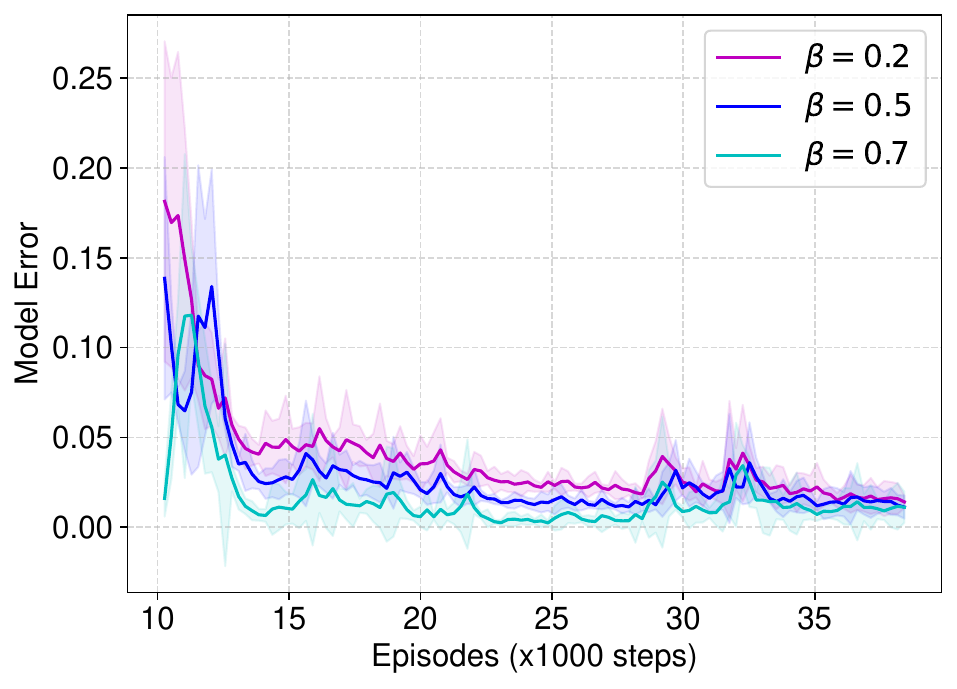}}
\caption{This figure shows the model error curves of MPDP with $\beta$ varying from 0.2 to 0.7, measured by the average $L_2$ norm of the predicted states on every 250 interactions. The model error decreases with $\beta$, which verifies that optimizing under our regularization effectively restricts behavior policy in the areas with low model error.}
\label{fig:beta_error}
\end{center}
\end{figure}

\paragraph{Model Error.}
We validate that the regularization based on OvR does push the policy to explore areas with low dynamic model error. We vary $\beta$ at \cref{eq:beta_obj}  with \{0.2, 0.5, 0.7\} and evaluate the model error as shown in \cref{fig:beta_error}. 
The result demonstrates that the model error decreases with $\beta$, which verifies the effectiveness of the designed regularization.
We also plot the final performance of corresponding $\beta$ in \cref{fig:beta_perform}. 
However, we find that a too large regularization harms the asymptotic performance due to the excessive restriction on the exploration area of the policy. \cref{fig:beta_perform} also implies that a larger regularization brings more stable results.

\begin{figure}[!h]
\begin{center}
\centerline{\includegraphics[width=0.47\textwidth]{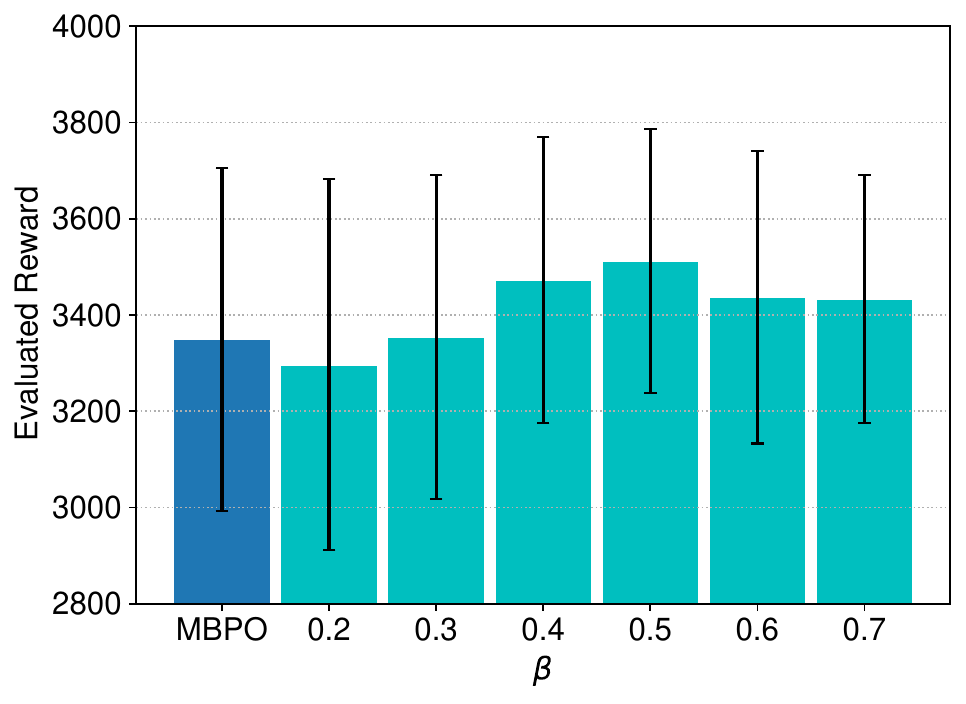}}
\caption{This figure displays the performance of MPDP with $\beta$ varying from 0.2 to 0.7 along with MBPO on the Hopper task, evaluated over 4 trials. As $\beta$ increases, the performance increases at first then decreases due to the too strong restriction on the exploration.
}
\label{fig:beta_perform}
\end{center}
\end{figure}

\section{Conclusion}
In this paper, we investigate the theoretical guarantee of distillation from model-based planning into an RL policy. We first extend the one-step optimization of SAC to a multi-step optimization formulation. Then, we develop a distillation approach based on the solution of the proposed multi-step optimization. It provably has the guarantee of monotonic improvement and convergence to the optimal policy. We further theoretically verify its potential to incorporate the foresight planning. Based on the theory, we discuss several design choices to instantiate a practical algorithm MPDP. 
Experimental results confirm that MPDP outperforms the state-of-the-art model-based planning algorithms in both sample efficiency and asymptotic performance on a range of continuous control tasks in MuJoCo.

One limitation is that the generalization ability of the horizon-adapted policy may not be strong enough because we fit the horizon to the model error for fast convergence speed. Thus, our method is efficient for task-specific but not exploration-oriented problems. We leave this to future work.

\newpage
\bibliographystyle{named}
\bibliography{ijcai23}

\begin{thebibliography}{}

\bibitem[\protect\citeauthoryear{Botev \bgroup \em et al.\egroup }{2013}]{cem}
Zdravko~I Botev, Dirk~P Kroese, Reuven~Y Rubinstein, and Pierre L’Ecuyer.
\newblock The cross-entropy method for optimization.
\newblock In {\em Handbook of statistics}, volume~31, pages 35--59. Elsevier,
  2013.

\bibitem[\protect\citeauthoryear{Buckman \bgroup \em et al.\egroup
  }{2018}]{steve}
Jacob Buckman, Danijar Hafner, George Tucker, Eugene Brevdo, and Honglak Lee.
\newblock Sample-efficient reinforcement learning with stochastic ensemble
  value expansion.
\newblock In {\em Proceedings of the 32nd International Conference on Neural
  Information Processing Systems}, pages 8234--8244, 2018.

\bibitem[\protect\citeauthoryear{Chua \bgroup \em et al.\egroup }{2018}]{pets}
Kurtland Chua, Roberto Calandra, Rowan McAllister, and Sergey Levine.
\newblock Deep reinforcement learning in a handful of trials using
  probabilistic dynamics models.
\newblock {\em Advances in Neural Information Processing Systems}, 31, 2018.

\bibitem[\protect\citeauthoryear{Clavera \bgroup \em et al.\egroup
  }{2020}]{maac}
Ignasi Clavera, Yao Fu, and Pieter Abbeel.
\newblock Model-augmented actor-critic: Backpropagating through paths.
\newblock In {\em 8th International Conference on Learning Representations},
  2020.

\bibitem[\protect\citeauthoryear{Ebert \bgroup \em et al.\egroup
  }{2018}]{VisualForesight}
Frederik Ebert, Chelsea Finn, Sudeep Dasari, Annie Xie, Alex~X. Lee, and Sergey
  Levine.
\newblock Visual foresight: Model-based deep reinforcement learning for
  vision-based robotic control.
\newblock {\em CoRR}, abs/1812.00568, 2018.

\bibitem[\protect\citeauthoryear{Feinberg \bgroup \em et al.\egroup
  }{2018}]{mve}
Vladimir Feinberg, Alvin Wan, Ion Stoica, Michael~I. Jordan, Joseph~E.
  Gonzalez, and Sergey Levine.
\newblock Model-based value estimation for efficient model-free reinforcement
  learning.
\newblock {\em CoRR}, abs/1803.00101, 2018.

\bibitem[\protect\citeauthoryear{Fujimoto \bgroup \em et al.\egroup
  }{2018}]{td3}
Scott Fujimoto, Herke Hoof, and David Meger.
\newblock Addressing function approximation error in actor-critic methods.
\newblock In {\em International Conference on Machine Learning}, pages
  1587--1596. PMLR, 2018.

\bibitem[\protect\citeauthoryear{Haarnoja \bgroup \em et al.\egroup
  }{2018}]{sac}
Tuomas Haarnoja, Aurick Zhou, Pieter Abbeel, and Sergey Levine.
\newblock Soft actor-critic: Off-policy maximum entropy deep reinforcement
  learning with a stochastic actor.
\newblock In {\em International conference on machine learning}, pages
  1861--1870. PMLR, 2018.

\bibitem[\protect\citeauthoryear{Heess \bgroup \em et al.\egroup }{2015}]{svg}
Nicolas Heess, Gregory Wayne, David Silver, Timothy Lillicrap, Tom Erez, and
  Yuval Tassa.
\newblock Learning continuous control policies by stochastic value gradients.
\newblock {\em Advances in Neural Information Processing Systems},
  28:2944--2952, 2015.

\bibitem[\protect\citeauthoryear{Hu \bgroup \em et al.\egroup }{2021}]{gem}
Hao Hu, Jianing Ye, Guangxiang Zhu, Zhizhou Ren, and Chongjie Zhang.
\newblock Generalizable episodic memory for deep reinforcement learning.
\newblock In {\em International Conference on Machine Learning}, pages
  4380--4390. PMLR, 2021.

\bibitem[\protect\citeauthoryear{Janner \bgroup \em et al.\egroup
  }{2019}]{mbpo}
Michael Janner, Justin Fu, Marvin Zhang, and Sergey Levine.
\newblock When to trust your model: Model-based policy optimization.
\newblock {\em Advances in Neural Information Processing Systems},
  32:12519--12530, 2019.

\bibitem[\protect\citeauthoryear{Jia \bgroup \em et al.\egroup }{2021}]{emc-ac}
Ruonan Jia, Qingming Li, Wenzhen Huang, Junge Zhang, and Xiu Li.
\newblock Consistency regularization for ensemble model based reinforcement
  learning.
\newblock In {\em Trends in Artificial Intelligence: 18th Pacific Rim
  International Conference on Artificial Intelligence, Proceedings, Part III
  18}, pages 3--16. Springer, 2021.

\bibitem[\protect\citeauthoryear{Kurutach \bgroup \em et al.\egroup
  }{2018}]{me-trpo}
Thanard Kurutach, Ignasi Clavera, Yan Duan, Aviv Tamar, and Pieter Abbeel.
\newblock Model-ensemble trust-region policy optimization.
\newblock In {\em International Conference on Learning Representations}, 2018.

\bibitem[\protect\citeauthoryear{Levine and Abbeel}{2014}]{gps}
Sergey Levine and Pieter Abbeel.
\newblock Learning neural network policies with guided policy search under
  unknown dynamics.
\newblock In {\em NIPS}, volume~27, pages 1071--1079. Citeseer, 2014.

\bibitem[\protect\citeauthoryear{Levine and Koltun}{2013}]{gps2}
Sergey Levine and Vladlen Koltun.
\newblock Guided policy search.
\newblock In {\em International conference on machine learning}, pages 1--9.
  PMLR, 2013.

\bibitem[\protect\citeauthoryear{Lillicrap \bgroup \em et al.\egroup
  }{2016}]{ddpg}
Timothy~P Lillicrap, Jonathan~J Hunt, Alexander Pritzel, Nicolas Heess, Tom
  Erez, Yuval Tassa, David Silver, and Daan Wierstra.
\newblock Continuous control with deep reinforcement learning.
\newblock In {\em ICLR (Poster)}, 2016.

\bibitem[\protect\citeauthoryear{Luo \bgroup \em et al.\egroup }{2019}]{slbo}
Yuping Luo, Huazhe Xu, Yuanzhi Li, Yuandong Tian, Trevor Darrell, and Tengyu
  Ma.
\newblock Algorithmic framework for model-based deep reinforcement learning
  with theoretical guarantees.
\newblock In {\em International Conference on Learning Representations}, 2019.

\bibitem[\protect\citeauthoryear{Mnih \bgroup \em et al.\egroup }{2013}]{dqn}
Volodymyr Mnih, Koray Kavukcuoglu, David Silver, Alex Graves, Ioannis
  Antonoglou, Daan Wierstra, and Martin Riedmiller.
\newblock Playing atari with deep reinforcement learning.
\newblock {\em arXiv preprint arXiv:1312.5602}, 2013.

\bibitem[\protect\citeauthoryear{Pan \bgroup \em et al.\egroup }{2020}]{m2ac}
Feiyang Pan, Jia He, Dandan Tu, and Qing He.
\newblock Trust the model when it is confident: Masked model-based
  actor-critic.
\newblock {\em Advances in neural information processing systems},
  33:10537--10546, 2020.

\bibitem[\protect\citeauthoryear{Rybkin \bgroup \em et al.\egroup
  }{2021}]{latco}
Oleh Rybkin, Chuning Zhu, Anusha Nagabandi, Kostas Daniilidis, Igor Mordatch,
  and Sergey Levine.
\newblock Model-based reinforcement learning via latent-space collocation.
\newblock In {\em International Conference on Machine Learning}, pages
  9190--9201. PMLR, 2021.

\bibitem[\protect\citeauthoryear{Schulman \bgroup \em et al.\egroup
  }{2017}]{ppo}
John Schulman, Filip Wolski, Prafulla Dhariwal, Alec Radford, and Oleg Klimov.
\newblock Proximal policy optimization algorithms.
\newblock {\em arXiv preprint arXiv:1707.06347}, 2017.

\bibitem[\protect\citeauthoryear{Tassa \bgroup \em et al.\egroup }{2012}]{mpc}
Yuval Tassa, Tom Erez, and Emanuel Todorov.
\newblock Synthesis and stabilization of complex behaviors through online
  trajectory optimization.
\newblock In {\em 2012 IEEE/RSJ International Conference on Intelligent Robots
  and Systems}, pages 4906--4913. IEEE, 2012.

\bibitem[\protect\citeauthoryear{Todorov \bgroup \em et al.\egroup
  }{2012}]{mujoco}
Emanuel Todorov, Tom Erez, and Yuval Tassa.
\newblock Mujoco: A physics engine for model-based control.
\newblock In {\em 2012 IEEE/RSJ International Conference on Intelligent Robots
  and Systems}, pages 5026--5033. IEEE, 2012.

\bibitem[\protect\citeauthoryear{Voelcker \bgroup \em et al.\egroup
  }{2022}]{vagram}
Claas~A Voelcker, Victor Liao, Animesh Garg, and Amir-massoud Farahmand.
\newblock Value gradient weighted model-based reinforcement learning.
\newblock In {\em International Conference on Learning Representations}, 2022.

\bibitem[\protect\citeauthoryear{Wang and Ba}{2019}]{poplin}
Tingwu Wang and Jimmy Ba.
\newblock Exploring model-based planning with policy networks.
\newblock In {\em International Conference on Learning Representations}, 2019.

\bibitem[\protect\citeauthoryear{Wang \bgroup \em et al.\egroup }{2023}]{pdml}
Xiyao Wang, Wichayaporn Wongkamjan, Ruonan Jia, and Furong Huang.
\newblock Live in the moment: Learning dynamics model adapted to evolving
  policy.
\newblock In {\em Proceedings of the 40th International Conference on Machine
  Learning}, volume 202 of {\em Proceedings of Machine Learning Research},
  pages 36470--36493. PMLR, 23--29 Jul 2023.

\bibitem[\protect\citeauthoryear{Zhang \bgroup \em et al.\egroup
  }{2022}]{gobigger}
Ming Zhang, Shenghan Zhang, Zhenjie Yang, Lekai Chen, Jinliang Zheng, Chao
  Yang, Chuming Li, Hang Zhou, Yazhe Niu, and Yu~Liu.
\newblock Gobigger: A scalable platform for cooperative-competitive multi-agent
  interactive simulation.
\newblock In {\em The Eleventh International Conference on Learning
  Representations}, 2022.

\bibitem[\protect\citeauthoryear{Zhou \bgroup \em et al.\egroup }{2022}]{taec}
Tong Zhou, Letian Wang, Ruobing Chen, Wenshuo Wang, and Yu~Liu.
\newblock Accelerating reinforcement learning for autonomous driving using
  task-agnostic and ego-centric motion skills.
\newblock {\em arXiv preprint arXiv:2209.12072}, 2022.

\end{thebibliography}

\newpage
\appendix
\onecolumn
\centerline{\textbf{\LARGE{Appendix: Theoretically Guaranteed Policy Improvement}}}
\vspace{10pt}
\centerline{\textbf{\LARGE{Distilled from Model-Based Planning}}}
~\\ 

\subsection*{A. Proof of Lemma and Theorem}

In this section, we provide proofs for bounds presented in the main paper.

\paragraph{\cref{lem:improve} (Policy Improvement).}
\label{apx:improve}
\textit{
Let $\pi_{\boldsymbol{s}_t}^H$ be the optimizer of the optimization objective of \cref{eq:objective}. 
When the new policy $\pi_{new}(\cdot|\boldsymbol{s}_t) = \pi_{\boldsymbol{s}_t}^H(\cdot|\boldsymbol{s}_t)$,
$V^{\pi_{new}}(\boldsymbol{s}_t) \geq V^{\pi_{old}}(\boldsymbol{s}_t)$ for all $\boldsymbol{s}_t \in S $.
}

\begin{proof} 
Before the proof, we need to show that
\begin{align}
\label{eq:v_old_new}
V^{\pi_{old}}(\boldsymbol{s}_t) \leq J_{\boldsymbol{s}_t}^H(\pi_{\boldsymbol{s}_t}^H),
\end{align}
because $\pi_{\boldsymbol{s}_t}^H$ is the optimal solution and $V^{\pi_{old}}(\boldsymbol{s}_t)=J_{\boldsymbol{s}_t}^H(\pi_{old})\leq J_{\boldsymbol{s}_t}^H(\pi_{\boldsymbol{s}_t}^H)$.

Next, we will prove that
\begin{align}
\label{eq:v_new_new}
J_{\boldsymbol{s}_t}^H(\pi_{\boldsymbol{s}_t}^H)\leq \mathbb{E}_{\boldsymbol{a}_{t} \sim \pi_{\boldsymbol{s}_t}^H} \left[ r^{\pi_{\boldsymbol{s}_t}^H}(\boldsymbol{s}_{t},\boldsymbol{a}_{t}) + \gamma \cdot J_{\boldsymbol{s}_{t+1}}^H(\pi_{\boldsymbol{s}_{t+1}}^H) \right],
\end{align}
which follows 

\begin{equation*}
\begin{aligned}
J_{\boldsymbol{s}_t}^H(\pi_{\boldsymbol{s}_t}^H)
&=
\mathbb{E}_{\boldsymbol{a}_{t:t+H-1} \sim \pi_{\boldsymbol{s}_t}^H} \left[
    r^{\pi_{\boldsymbol{s}_t}^H}(\boldsymbol{s}_{t},\boldsymbol{a}_{t}) + \cdots
    + 
    \gamma^{H-1}
    \cdot
    r^{\pi_{\boldsymbol{s}_t}^H}(\boldsymbol{s}_{t+H-1},\boldsymbol{a}_{t+H-1})
    + 
    \gamma^H
    \cdot
    V^{\pi_{old}}(\boldsymbol{s}_{t+H}) 
\right]\\
&=
\mathbb{E}_{\boldsymbol{a}_{t:t+H-1} \sim \pi_{\boldsymbol{s}_t}^H,
\boldsymbol{a}_{t+H} \sim \pi_{old}} \Big[  
    r^{\pi_{\boldsymbol{s}_t}^H}(\boldsymbol{s}_{t},\boldsymbol{a}_{t}) + \cdots
    + 
    \gamma^{H-1}
    \cdot
    r^{\pi_{\boldsymbol{s}_t}^H}(\boldsymbol{s}_{t+H-1},\boldsymbol{a}_{t+H-1}) \\
    &\phantom{=\;\;}+
    \gamma^{H}
    \cdot
    r^{\pi_{old}}(\boldsymbol{s}_{t+H},\boldsymbol{a}_{t+H}) 
    + 
    \gamma^{H+1}
    \cdot
    V^{\pi_{old}}(\boldsymbol{s}_{t+H+1}) 
\Big]\\
&=
\mathbb{E}_{
\boldsymbol{a}_{t} \sim \pi_{\boldsymbol{s}_t}^H,
\boldsymbol{a}_{t+1:t+H-1} \sim \pi_{\boldsymbol{s}_t}^H,
\boldsymbol{a}_{t+H} \sim \pi_{old}
} \Big[
    r^{\pi_{\boldsymbol{s}_t}^H}(\boldsymbol{s}_{t},\boldsymbol{a}_{t}) + 
    \gamma
    \cdot
    \left[ \right.
    r^{\pi_{\boldsymbol{s}_{t}}^H}(\boldsymbol{s}_{t+1},\boldsymbol{a}_{t+1}) + \cdots
    + 
    \gamma^{H-2}
    \cdot
    r^{\pi_{\boldsymbol{s}_t}^H}(\boldsymbol{s}_{t+H-1},\boldsymbol{a}_{t+H-1}) \\
    &\phantom{=\;\;}+
    \gamma^{H-1}
    \cdot
    r^{\pi_{old}}(\boldsymbol{s}_{t+H},\boldsymbol{a}_{t+H}) 
    + 
    \gamma^{H}
    \cdot
    V^{\pi_{old}}(\boldsymbol{s}_{t+H+1}) 
    \left.\right]
\Big]\\
&\leq
\mathbb{E}_{
\boldsymbol{a}_{t} \sim \pi_{\boldsymbol{s}_t}^H,
\boldsymbol{a}_{t+1:t+H} \sim \pi_{\boldsymbol{s}_{t+1}}^H
} \Big[
    r^{\pi_{\boldsymbol{s}_t}^H}(\boldsymbol{s}_{t},\boldsymbol{a}_{t}) + 
    \gamma
    \cdot
    \left[ \right.
    r^{\pi_{\boldsymbol{s}_{t+1}}^H}(\boldsymbol{s}_{t+1},\boldsymbol{a}_{t+1}) + \cdots 
    + 
    \gamma^{H-1}
    \cdot
    r^{\pi_{\boldsymbol{s}_{t+1}}^H}(\boldsymbol{s}_{t+H-1},\boldsymbol{a}_{t+H-1}) \\
    &\phantom{=\;\;}+ 
    \gamma^{H}
    \cdot
    V^{\pi_{old}}(\boldsymbol{s}_{t+H+1}) 
    \left.\right]
\Big]\\
&=\mathbb{E}_{\boldsymbol{a}_{t} \sim \pi_{\boldsymbol{s}_t}^H} \left[ r^{\pi_{\boldsymbol{s}_t}^H}(\boldsymbol{s}_{t},\boldsymbol{a}_{t}) + \gamma \cdot J_{\boldsymbol{s}_{t+1}}^H(\pi_{\boldsymbol{s}_{t+1}}^H) \right].
\end{aligned}
\end{equation*}

We finish the proof by applying \cref{eq:v_old_new} and iteratively applying \cref{eq:v_new_new}:

\begin{align*}
 V^{\pi_{old}}(\boldsymbol{s}_t) &\leq J_{\boldsymbol{s}_t}^H(\pi_{\boldsymbol{s}_t}^H)\\
 &\leq 
 \mathbb{E}_{\boldsymbol{a}_{t} \sim \pi_{\boldsymbol{s}_t}^H} \left[ r^{\pi_{\boldsymbol{s}_t}^H}(\boldsymbol{s}_{t},\boldsymbol{a}_{t}) 
 + \gamma \cdot J_{\boldsymbol{s}_{t+1}}^H(\pi_{\boldsymbol{s}_{t+1}}^H) \right]
 \\
 &\leq 
 \mathbb{E}_{\boldsymbol{a}_{t} \sim \pi_{\boldsymbol{s}_t}^H,
 \boldsymbol{a}_{t+1} \sim \pi_{\boldsymbol{s}_{t+1}}^H} \left[ r^{\pi_{\boldsymbol{s}_t}^H}(\boldsymbol{s}_{t},\boldsymbol{a}_{t}) 
 + \gamma \cdot r^{\pi_{\boldsymbol{s}_{t+1}}^H}(\boldsymbol{s}_{t+1},\boldsymbol{a}_{t+1})
 \right]\\
 &\vdots\\
 &\leq 
 \mathbb{E}_{\boldsymbol{a}_{t} \sim \pi_{new}} \Big[ r^{\pi_{new}}(\boldsymbol{s}_{t},\boldsymbol{a}_{t}) +...
 \Big]\\
 &= V^{\pi_{new}}(\boldsymbol{s}_t).
\end{align*}
\end{proof}

\paragraph{\cref{thm:policy_conv} (Policy Convergence).}
\label{apx:policy_conv}
\textit{
Let $\pi_0$ be any initial policy. Assuming $|A|<\infty$, if the policy evaluation in \cref{eq:sac_evaluation} and the policy improvement with the objective in \cref{eq:objective} are alternatively carried out, $\pi_0$ converges to a policy $\pi_*$, with $V^{\pi_{*}}(\boldsymbol{s}_t) \geq V^{\pi}(\boldsymbol{s}_t)$ for any $\boldsymbol{s}_t \in S $.
}
\begin{proof}
First, let $\pi_{i}$ be the policy at the $i$-th iteration. Because $V^{\pi_{i}}(\boldsymbol{s}_t)$ monotonically increases with $i$ and is bounded, the sequence $\pi_{i}$ converges to some $\pi_{*}$.

We will next prove that, when the old policy $\pi_{old}=\pi_{*}$, $V^{\pi_{*}}(\boldsymbol{s}_t) = J_{\boldsymbol{s}_t}^H(\pi_{\boldsymbol{s}_t}^H)$. First, because $\pi_{\boldsymbol{s}_t}^H$ is the optimal solution of $J_{\boldsymbol{s}_t}^H(\pi_{\boldsymbol{s}_t}^H)$, as shown in the proof of \cref{lem:improve}, $V^{\pi_{*}}(\boldsymbol{s}_t)=V^{\pi_{old}}(\boldsymbol{s}_t) \leq J_{\boldsymbol{s}_t}^H(\pi_{\boldsymbol{s}_t}^H)$. Secondly, because $\pi_{*}$ is the fixed point, $\pi_{new}=\pi_{old}=\pi_{*}$ and $V^{\pi_{*}}(\boldsymbol{s}_t)=V^{\pi_{new}}(\boldsymbol{s}_t) \geq J_{\boldsymbol{s}_t}^H(\pi_{\boldsymbol{s}_t}^H)$, which completes the proof.

Finally, let $\pi$ be any other policy with $\pi \neq \pi_{*}$. We have $V^{\pi_{*}}(\boldsymbol{s}_t)=J_{\boldsymbol{s}_t}^H(\pi_{\boldsymbol{s}_t}^H) \geq J_{\boldsymbol{s}_t}^H(\pi)$ and expand the inequality as:

\begin{align*}
V^{\pi_{*}}(\boldsymbol{s}_t)
&\geq 
J_{\boldsymbol{s}_t}^H(\pi) 
=
\mathbb{E}_{
\boldsymbol{a}_{t:t+H-1} \sim \pi
}
 \Big[ 
    r^{\pi}(\boldsymbol{s}_{t},\boldsymbol{a}_{t}) + 
    \cdots
    + 
    \gamma^{H-1}
    \cdot
    r^{\pi}(\boldsymbol{s}_{t+H-1},\boldsymbol{a}_{t+H-1}) 
    + 
    \gamma^{H}
    \cdot
    V^{\pi_{*}}(\boldsymbol{s}_{t+H}) 
\Big]\\
&\geq
\mathbb{E}_{
\boldsymbol{a}_{t:t+2H-1} \sim \pi
}
 \Big[ 
    r^{\pi}(\boldsymbol{s}_{t},\boldsymbol{a}_{t}) + 
    \cdots
    + 
    \gamma^{2H-1}
    \cdot
    r^{\pi}(\boldsymbol{s}_{t+2H-1},\boldsymbol{a}_{t+2H-1}) 
    + 
    \gamma^{2H}
    \cdot
    V^{\pi_{*}}(\boldsymbol{s}_{t+2H}) 
\Big]\\
&\vdots\\
&\geq
\mathbb{E}_{
\boldsymbol{a}_{t:\infty} \sim \pi
}
 \Big[ 
    r^{\pi}(\boldsymbol{s}_{t},\boldsymbol{a}_{t}) + 
    \cdots 
\Big]\\
&=V^{\pi}(\boldsymbol{s}_t).
\end{align*}

\end{proof}

\paragraph{\cref{lem:monotone} (Policy Monotone with Horizon).}
\label{apx:h_monotone}
\textit{
Let $\pi_{\boldsymbol{s}_t}^H$ and $\pi_{\boldsymbol{s}_t}^{H+1}$ be the optimizer of $J_{\boldsymbol{s}_t}^H(\pi)$ and $J_{\boldsymbol{s}_t}^{H+1}(\pi)$. Then $J_{\boldsymbol{s}_t}^{H+1}(\pi_{\boldsymbol{s}_t}^{H+1}) \geq J_{\boldsymbol{s}_t}^{H}(\pi_{\boldsymbol{s}_t}^{H})$ for all $H\geq 1$ and $\boldsymbol{s}_t \in S$.
}
\begin{proof}
\begin{align*}
J_{\boldsymbol{s}_t}^H(\pi_{\boldsymbol{s}_t}^H)
&=
\mathbb{E}_{\boldsymbol{a}_{t:t+H-1} \sim \pi_{\boldsymbol{s}_t}^H} \left[
    r^{\pi_{\boldsymbol{s}_t}^H}(\boldsymbol{s}_{t},\boldsymbol{a}_{t}) + \cdots
    + 
    \gamma^{H-1}
    \cdot
    r^{\pi_{\boldsymbol{s}_t}^H}(\boldsymbol{s}_{t+H-1},\boldsymbol{a}_{t+H-1}) 
    + 
    \gamma^H
    \cdot
    V^{\pi_{old}}(\boldsymbol{s}_{t+H}) 
\right]\\
&=
\mathbb{E}_{\boldsymbol{a}_{t:t+H-1} \sim \pi_{\boldsymbol{s}_t}^H,
\boldsymbol{a}_{t+H} \sim \pi_{old}} \Big[
    r^{\pi_{\boldsymbol{s}_t}^H}(\boldsymbol{s}_{t},\boldsymbol{a}_{t}) + \cdots
    + 
    \gamma^{H-1}
    \cdot
    r^{\pi_{\boldsymbol{s}_t}^H}(\boldsymbol{s}_{t+H-1},\boldsymbol{a}_{t+H-1}) \\
    &\phantom{=\;\;}+ 
    \gamma^{H}
    \cdot
    r^{\pi_{old}}(\boldsymbol{s}_{t+H},\boldsymbol{a}_{t+H}) 
    + 
    \gamma^{H+1}
    \cdot
    V^{\pi_{old}}(\boldsymbol{s}_{t+H+1}) 
\Big]\\
&\leq
\mathbb{E}_{\boldsymbol{a}_{t:t+H} \sim \pi_{\boldsymbol{s}_t}^{H+1}} \left[
    r^{\pi_{\boldsymbol{s}_t}^{H+1}}(\boldsymbol{s}_{t},\boldsymbol{a}_{t}) + \cdots
    + 
    \gamma^{H}
    \cdot
    r^{\pi_{\boldsymbol{s}_t}^{H+1}}(\boldsymbol{s}_{t+H},\boldsymbol{a}_{t+H}) 
    + 
    \gamma^{H+1}
    \cdot
    V^{\pi_{old}}(\boldsymbol{s}_{t+H+1}) 
\right]\\
&=
J_{\boldsymbol{s}_t}^{H+1}(\pi_{\boldsymbol{s}_t}^{H+1}).
\end{align*}
\end{proof}

\paragraph{\cref{thm:optim} (Policy Convergence with Horizon).}
\label{apx:optim}
\textit{
Let $\pi_{\boldsymbol{s}_t}^H$ be the optimal solution of $J_{\boldsymbol{s}_t}^H(\pi)$, and $\pi_{new}(\cdot|\boldsymbol{s}_t)=\pi_{\boldsymbol{s}_t}^H(\cdot|\boldsymbol{s}_t)$. $\pi_{*}$ denotes the optimal policy. As $H$ increases, $V^{\pi_{new}}$ and $J_{\boldsymbol{s}_t}^H(\pi_{\boldsymbol{s}_t}^H)$ converge to $V^{\pi_{*}}$ for all $\boldsymbol{s}_t \in S$. Specifically, $J_{\boldsymbol{s}_t}^H(\pi_{\boldsymbol{s}_t}^H) \geq V^{\pi_{*}}(\boldsymbol{s}_{t}) 
-    
\frac{\gamma^H
\cdot r^{max}}{1-\gamma}$ with $r^{max}$ the maximum of $r^{\pi}(\boldsymbol{s},\boldsymbol{a})$ over all $\pi$ and $(\boldsymbol{s},\boldsymbol{a})\in |\mathcal{S}|\times|\mathcal{A}|$.
}
\begin{proof}
We have show that $V^{\pi_{new}}\geq J_{\boldsymbol{s}_t}^H(\pi_{\boldsymbol{s}_t}^H)$ in the proof of \cref{lem:improve}, hence we only need to prove that $J_{\boldsymbol{s}_t}^H(\pi_{\boldsymbol{s}_t}^H)$ converges to $V^{\pi_{*}}$. We start the proof with the fact that $J_{\boldsymbol{s}_t}^H(\pi_{\boldsymbol{s}_t}^H) \geq J_{\boldsymbol{s}_t}^H(\pi_{*})$ and expand this inequality as:
\begin{align*}
    J_{\boldsymbol{s}_t}^H(\pi_{\boldsymbol{s}_t}^H) &\geq
\mathbb{E}_{\boldsymbol{a}_{t:t+H-1} \sim \pi_{*}} \Big[
    r^{\pi_{*}}(\boldsymbol{s}_{t},\boldsymbol{a}_{t}) + \cdots 
    + 
    \gamma^{H-1}
    \cdot
    r^{\pi_{*}}(\boldsymbol{s}_{t+H-1},\boldsymbol{a}_{t+H-1}) 
    + 
    \gamma^H
    \cdot
    V^{\pi_{old}}(\boldsymbol{s}_{t+H}) 
\Big]\\
&=
\mathbb{E}_{\boldsymbol{a}_{t:t+H-1} \sim \pi_{*}} \Big[
    r^{\pi_{*}}(\boldsymbol{s}_{t},\boldsymbol{a}_{t}) + \cdots 
    + 
    \gamma^{H-1}
    \cdot
    r^{\pi_{*}}(\boldsymbol{s}_{t+H-1},\boldsymbol{a}_{t+H-1}) \\
    &\phantom{=\;\;}+ 
    \gamma^H
    \cdot
    V^{\pi_{*}}(\boldsymbol{s}_{t+H})
    +
    \gamma^H
    \cdot
    V^{\pi_{old}}(\boldsymbol{s}_{t+H})
    -
    \gamma^H
    \cdot
    V^{\pi_{*}}(\boldsymbol{s}_{t+H}) 
\Big]\\
&=
V^{\pi_{*}}(\boldsymbol{s}_{t}) 
-    
\gamma^H
\cdot 
\mathbb{E}_{\boldsymbol{a}_{t:t+H-1} \sim \pi_{*}} \Big[ 
    V^{\pi_{*}}(\boldsymbol{s}_{t+H}) -
    V^{\pi_{old}}(\boldsymbol{s}_{t+H})
\Big]\\
&\geq
V^{\pi_{*}}(\boldsymbol{s}_{t}) 
-    
\gamma^H
\cdot 
\Vert 
V^{\pi_{*}} -
    V^{\pi_{old}}
\Vert_{\infty},\\
&\geq
V^{\pi_{*}}(\boldsymbol{s}_{t}) 
-    
\frac{\gamma^H
\cdot r^{max}}{1-\gamma}.
\end{align*}
\end{proof}


\subsection*{B. Implementation}
\subsection*{B.1 Experiment Setup}

We implement MPDP based on the open-source platform DI-engine \footnote{https://github.com/opendilab/DI-engine}.
And \cref{tab:hyparam} provides the key hyperparameters in MPDP.
We follow the original implementations for all baseline algorithms with regard to the reward sum over 1000 steps.
We evaluate MPDP along with the baseline algorithms on six continuous control tasks provided in MuJoCo-v2 \cite{mujoco}.

\begin{table*}[ht]
  \centering
  \begin{tabular}{cc}
    \toprule
    Hyperparameter  & Value \\ 
    \midrule
    Ensemble size    & 7 \\
    Replay buffer size  & $10^6$  \\
    Batch size    & 256 \\
    Learning rate & $3 \cdot 10^{-4}$ \\
    Threshold $u_T$    & -5 \\
    Entropy coefficient $\alpha$    & 0.2 \\
    Regularization coefficient $\beta$    & 0.5 \\
    Maximum horizon $H_{max}$  & 25 \\
    Policy updates per environment step  & 20  \\
    Environment steps per model training   & 250 \\
    \bottomrule
  \end{tabular}
  \caption{Hyperparameter setup for MPDP.}
  \label{tab:hyparam}
\end{table*}

\subsection*{B.2 Experiment Environments}

We visualize the six continuous control tasks in MuJoCo-v2 including  InvertedPendulum, Hopper, HalfCheetah, Ant, Walker2d, and Humanoid, as shown in \cref{supp:envs}. 
The first task InvertedPendulum is designed to control the pole to keep balance, and the other five tasks aim to keep the agent moving forward without falling.

\begin{figure*}[!htbp]
\centering
\subfigure[InvertedPendulum]{
\includegraphics[width=4.3cm]{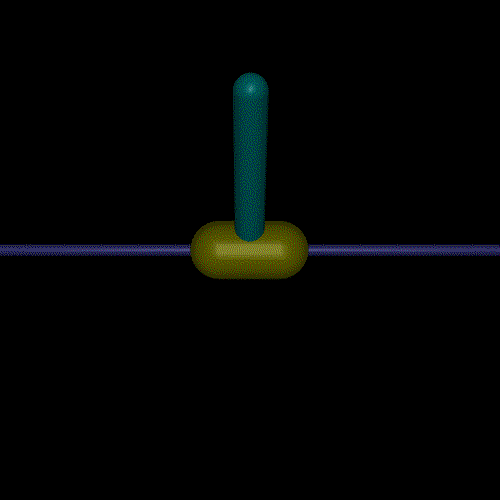}}
\subfigure[Hopper]{
\includegraphics[width=4.3cm]{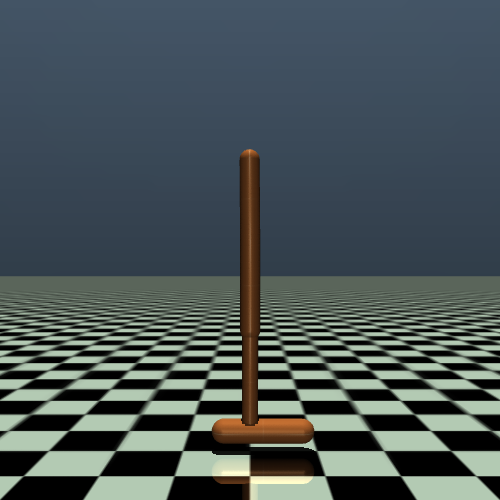}}
\subfigure[HalfCheetah]{
\includegraphics[width=4.3cm]{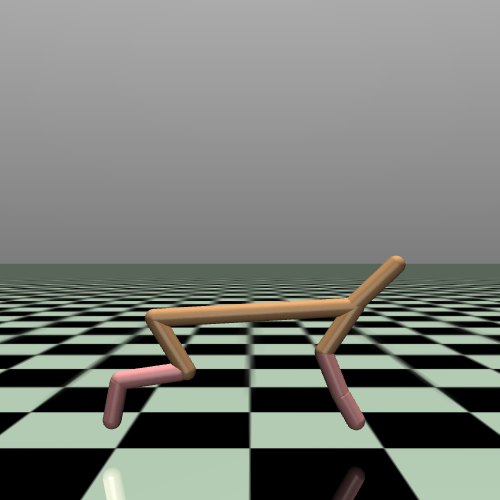}}
\subfigure[Ant]{
\includegraphics[width=4.3cm]{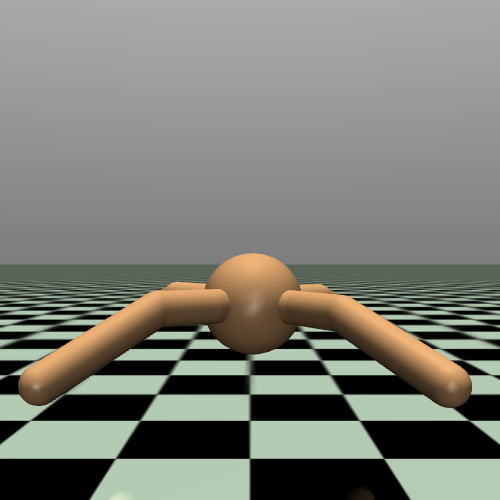}}
\subfigure[Walker2d]{
\includegraphics[width=4.3cm]{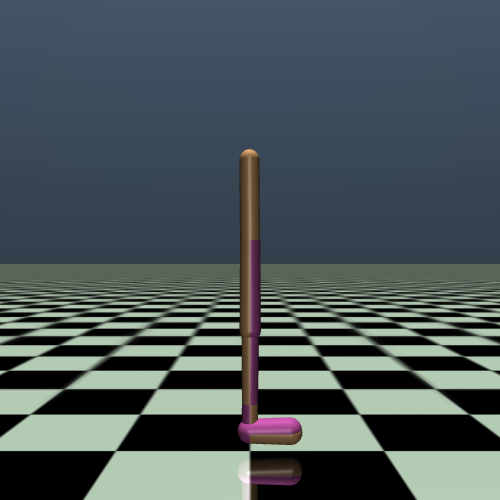}}
\subfigure[Humanoid]{
\includegraphics[width=4.3cm]{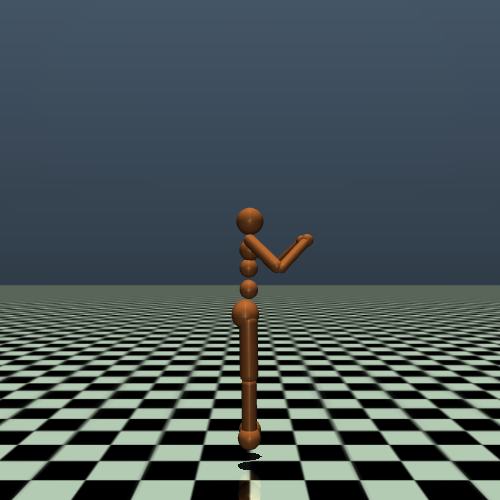}}
\caption{The screenshots of MuJoCo-v2 simulation environments used in our experiments.}
\label{supp:envs}
\end{figure*}

\end{document}